\documentclass{article}

\usepackage{hyperref}
\usepackage{amsmath}
\usepackage{amssymb}
\usepackage{amsthm}
\usepackage{tikz}
\usepackage{pgfplots}
\usepackage{caption}

\pgfplotsset{compat=1.4} 

\newcounter{global}
\theoremstyle{definition}
\newtheorem{definition}[global]{Definition}
\theoremstyle{plain}
\newtheorem{theorem}[global]{Theorem}
\newtheorem{proposition}[global]{Proposition}
\newtheorem{lemma}[global]{Lemma}
\newtheorem{corollary}[global]{Corollary}
\newtheoremstyle{note}{}{}{}{}{\itshape}{.}{.5em}{}
\theoremstyle{note}
\newtheorem{remark}{Remark}%
\newtheorem{example}{Example}%
\renewcommand\section{%
  \@startsection {section}{1}{\z@}%
  {-3.5ex \@plus -1ex \@minus -.2ex}%
  {2.3ex \@plus.2ex}%
  {\normalfont\large\bfseries}}
\def\itm#1{{\rm(\textit{\romannumeral#1})}}
\newcommand{\univ}[1][\Gamma]{\ensuremath{\mathrm{Th}(#1,Y)}}
\def\lexlt{\ensuremath{\mathrel{\lhd}}}
\begin{document}

\title{On sets of graded attribute implications with witnessed non-redundancy}

\date{\normalsize%
  Dept. Computer Science, Palacky University Olomouc}

\author{Vilem Vychodil\footnote{%
    e-mail: \texttt{vychodil@binghamton.edu},
    phone: +420 585 634 705,
    fax: +420 585 411 643}}

\maketitle

\begin{abstract}
  We study properties of particular non-redundant sets of if-then rules
  describing dependencies between graded attributes. We introduce notions of
  saturation and witnessed non-redundancy of sets of graded attribute
  implications are show that bases of graded attribute implications given
  by systems of pseudo-intents correspond to non-redundant sets of graded
  attribute implications with saturated consequents where the non-redundancy
  is witnessed by antecedents of the contained graded attribute implications.
  We introduce an algorithm which transforms any complete set
  of graded attribute implications parameterized by globalization into
  a base given by pseudo-intents. Experimental evaluation is provided to
  compare the method of obtaining bases for general parameterizations
  by hedges with earlier graph-based approaches.
\end{abstract}

\section{Introduction}
In this paper, we introduce the notion of a witnessed non-redundancy of sets of
graded attribute implications, study its properties and its relationship to
the notion of a general system of pseudo-intents which has been
introduced earlier~\cite{BeChVy:Ifdwfa}.
The graded attribute implications (also known
as fuzzy attribute implications) are if-then rules which generalize the
ordinary attribute implications which appear in formal concept
analysis~\cite{GaWi:FCA}. The graded attribute implications are more general
formulas than the classic attribute implications in that they allow to express
attribute dependencies to degrees. For instance, a rule 
\begin{align}
  \{
  {}^{0.9\!}/\text{\emph{good neighborhood}},
  {}^{1\!}/\text{\emph{large}}\} \Rightarrow
  \{{}^{0.98\!}/\text{\emph{expensive}}\}
\end{align}
saying that if an object (e.g., a house for sale) is located in
a good neighborhood and is large, then it is expensive, may be seen as
a typical example of a graded attribute implication. In this example,
the values $0.9$, $1$, and $0.98$ (taken from the real unit interval)
express lower bounds (or thresholds) of truth degrees to which we
consider the attributes valid in data.
Therefore, a finer reading of the rule is: ``if a house is located in a good
neighborhood \emph{at least to degree $0.9$} and is large
\emph{at least to degree $1$},
then it is expensive \emph{at least to degree $0.98$}''. In formal concept analysis
(FCA) of graded object-attribute data and in particular in the approach
to FCA with linguistic hedges~\cite{BeVy:Fcalh}, graded attribute implications
play an analogous role as the classic attribute implications in the ordinary FCA.

In FCA, one typically wants to find a small representative set of attribute
implications which conveys the information about all attribute implications which
hold is a given formal context. Equivalently, one wishes to find a small set
of attribute implications whose models are exactly all concept intents of
the data. Guigues-Duquenne bases~\cite{GuDu} which are determined by
pseudo-intents of formal concepts are examples of such
sets which are in addition minimal in terms of the number of contained formulas,
cf. also~\cite{Ga:Tbaca}.
In FCA with graded attributes, a general notion of a system of pseudo-intents has
been proposed and studied, see~\cite{BeVy:ADfDwG} for a survey. Unlike the classic
case, general systems of pseudo-intents are not unique and may not ensure
minimality of the corresponding base. Also, such systems may not exist when
the structure of degrees is infinite and their existence for general finite
scales is an open problem. From the computational point of view, graph-theoretic
methods for obtaining general systems of pseudo-intents are proposed but they are
limited only to small data sets. Therefore, further investigation is needed
and this paper makes a contribution to this area.

In this paper, we show that bases of graded attribute implications given by
general systems of pseudo-intents correspond to non-redundant sets of
graded attribute implications with saturated consequents where the
non-redundancy of each formula in the set is witnessed by its own antecedent.
Both the notions of \emph{saturation} and \emph{witnessed non-redundancy}
are introduced in Section~\ref{sec:results}. Furthermore, we introduce
a constructive method for transformation of any set of graded attribute
implications (which is complete in a given data) to a non-redundant base
with witnessed non-redundancy from which a system of pseudo-intents
can be derived. In practice, this means that we can avoid the graph-based
method and compute systems of pseudo-intents by an alternative and much
faster approach. We prove that the proposed procedure works if we consider
globalization~\cite{TaTi:Gist} as a parameter of the interpretation of
graded attribute implications. For linguistic
hedges~\cite{BeVy:Fcalh,EsGoNo:Hedges,Haj:Ovt} other
than the globalization,
which serve as parameters of the interpretation of graded attribute implications,
the procedure may not produce the desired base but as our experimental
observations show, it seems to have a high success rate.

The results contained in our paper fall in the category of results on
bases of if-then rules generated from
data~\cite{BaOb:Ocdgbi,BeMo:Tmfcduib,LoBeCoEnMo:Tdobvr}
which develop ideas of the seminal paper~\cite{GuDu}. Although we
work with graded if-then rules with semantics defined using complete
residuated lattices as structures of degrees and parameterized by
linguistic hedges, our approach is general and we anticipate it can be
adopted in recently developed approaches
such as~\cite{KoMeOjAc:Maclhch,MeOjAc:Dmacl}.

Our paper is organized as follows. In Section~\ref{sec:prelim}, we present the
basic notions of residuated structures of truth degrees and graded attribute
implications. In Section~\ref{sec:bases_survey}, we present a background and
a survey of existing results on non-redundant bases of graded attribute
implications and give further motivation for our work.
Section~\ref{sec:results} contains the new results.
Finally, Section~\ref{sec:notes} shows experimental observations on efficiency
on computing sets of graded attribute implications with witnessed non-redundancy
and presents open problems.

\section{Preliminaries}\label{sec:prelim}
In this section, we present the basic notions of structures of truth degrees
and graded attribute implications. Whenever possible, we keep the same notation
as in~\cite{Bel:FRS} for general residuated structures and~\cite{BeVy:ADfDwG}
for graded attribute implications.

We utilize complete residuated lattices as structures of truth degrees.
For our development, these structures represent a reasonable generalization
of the most common structures of degrees defined on the real unit interval
using left-continuous triangular norms~\cite{EsGo:MTL,KMP:TN}.
Recall that a complete residuated lattice \cite{Bel:FRS,GaJiKoOn:RL}
is an algebra
$\mathbf{L}=\langle L,\wedge,\vee,\otimes,\rightarrow,0,1\rangle$ where
$\langle L,\wedge,\vee,0,1 \rangle$ is
a complete lattice (i.e., a lattice where infima and suprema exist
for arbitrary subsets of~$L$), $\langle L,\otimes,1 \rangle$ is
a commutative monoid (i.e., $\otimes$ is commutative, associative,
and $1$ is neutral with respect to~$\otimes$),
and $\otimes$ and $\rightarrow$ satisfy the so-called
adjointness property: for all $a,b,c \in L$,
we have that $a \otimes b \leq c$ if{}f $a \leq b \rightarrow c$,
where $\leq$ is the (complete lattice) order induced by $\mathbf{L}$
(i.e., $a \leq b$ if{}f $a = a \wedge b$ if{}f $a \vee b = b$
if{}f $a \rightarrow b = 1$).
We interpret $\otimes$ and $\rightarrow$ as it is usual in
mathematical fuzzy logics~\cite{CiHa:Tnbpfl,Gog:Lic,Got:Mfl,Haj:MFL}
and their applications~\cite{KlYu}:
$\otimes$ is a truth function of ``fuzzy conjunction'' and $\rightarrow$
is a truth function of ``fuzzy implication'',
cf. also~\cite{CiHaNo1,CiHaNo2} for surveys of results on
fuzzy logics in the narrow sense.

In the paper, we use illustrative examples based on finite (and thus complete)
residuated lattices defined on equidistant subchains of the real unit interval.
That is, we consider $L = \bigl\{0,\frac{1}{n},\frac{2}{n},\ldots,1\bigr\}$ for
some natural $n$ and use the natural ordering of rational numbers, i.e.,
$\wedge$ and $\vee$ coincide with the operations of minimum and maximum,
respectively. If $\otimes$ coincides with $\wedge$, we call the resulting
$\mathbf{L}$ a finite G\"odel chain in which case we have $a \rightarrow b = 1$
if{}f $a \leq b$ and $a \rightarrow b = b$ otherwise. If $\otimes$ and $\rightarrow$
are given by
\begin{align}
  \textstyle \frac{i}{n} \otimes \frac{j}{n} &=
  \textstyle \max\bigl\{0, \frac{i}{n} + \frac{j}{n} - 1\bigr\},
  \\
  \textstyle \frac{i}{n} \rightarrow \frac{j}{n} &=
  \textstyle \min\bigl\{1, 1 - \frac{i}{n} + \frac{j}{n}\bigr\},
\end{align}
we call the resulting $\mathbf{L}$
a finite \L ukasiewicz chain. More general finite residuated lattices
on equidistant subchains of the real unit interval may be considered but in our
examples we utilize only these two basic structures, cf.~\cite{DeMe:TNPL,KMP:TN}.

In addition to $\otimes$ and $\rightarrow$ which may be seen as generalizations
of truth function of the classic logical connectives ``conjunction'' and ``implication'',
we make use of linguistic hedges~\cite{Za:Afstilh,Za:lv1,Za:lv2,Za:lv3} which do
not have nontrivial counterparts in classic logics. In particular, we utilize
idempotent truth-stressing (i.e., truth intensifying) linguistic hedges
(shortly, hedges), which are considered as maps ${}^*\!: L \to L$ such that 
$1^* = 1$, $a^* \leq a$,
$(a \rightarrow b)^* \leq a^* \rightarrow b^*$, and
$a^* \leq a^{**}$ for all $a,b \in L$.
Using similar arguments as in~\cite{Haj:Ovt}, such maps
may be seen as truth functions of logical connectives ``very true'',
cf. also \cite{EsGoNo:Hedges} and~\cite{CiHaNo2} for recent results on hedges.
Two basic hedges can be introduced on any complete residuated lattice. Namely,
(i) the identity (i.e., $a^* = a$ for all $a \in L$),
and (ii) the so-called globalization~\cite{TaTi:Gist}:
\begin{align}
  a^{\ast} = \left\{
    \begin{array}{@{\,}l@{\quad}l@{}}
      1, & \text{if } a = 1, \\
      0, & \text{otherwise,}
    \end{array}
  \right.
  \label{eqn:glob}
\end{align}
for all $a \in L$. Note that on linear residuated lattices,
the globalization coincides with the truth function of
the Baaz~$\Delta$ connective \cite{Baaz}.

For a fixed complete residuated lattice $\mathbf{L}$ and
a non-empty universe set $Y$, an $\mathbf{L}$-set $A$ in $Y$
(or an $\mathbf{L}$-fuzzy set~\cite{Gog:LFS}) is
any map $A\!: Y \to L$. As usual, $A(y)$ is interpreted as the degree
to which $y$ belongs to $A$.
The collection of all $\mathbf{L}$-sets in $Y$ is denoted by $L^Y$.
In a similar fashion, we introduce binary
$\mathbf{L}$-relations: for non-empty universe sets $X$ and $Y$,
a (binary) $\mathbf{L}$-relation between $X$ and $Y$
(or an $\mathbf{L}$-fuzzy relation between $X$ and $Y$)
is any map $R\!: X \times Y \to L$ with $R(x,y)$ understood
as the degree to which $x$ and $y$ are $R$-related (or related by $R$).
It is convenient to treat binary $\mathbf{L}$-relations 
between $X$ and $Y$ as $\mathbf{L}$-set in $X \times Y$.
We write $\mathbf{L}$-sets and $\mathbf{L}$-relation on finite
universes in the usual way, i.e.,
$\{{}^{a_1\!}/y_1,\ldots,{}^{a_n\!}/y_n\}$ denotes an $\mathbf{L}$-set $A$
in $Y = \{y_1,\ldots,y_n\}$ such that $A(y_i) = a_i$ for all $i=1,\ldots,n$.
Optionally, we omit ${}^{a_i\!}/y_i$ whenever $a_i = 0$ and write
just $y_i$ instead of ${}^{a_i\!}/y_i$ whenever $a_i = 1$.
In particular, $\{\}$ denotes the empty $\mathbf{L}$-set in $Y$,
i.e., $\{\}(y) = 0$ for all $y \in Y$.
The basic operations with $\mathbf{L}$-sets are defined componentwise using
operations in $\mathbf{L}$.
For instance, if $A$ and $B$ are $\mathbf{L}$-sets in $Y$,
then $A \cap B$ denotes the $\mathbf{L}$-set in $Y$ (called the intersection
of $A$ and~$B$) such that $(A \cap B)(y) = A(y) \wedge B(y)$ for each $y \in Y$
and analogously for $\cup$ and $\vee$.

We utilize the notion of a graded subsethood~\cite{Gog:LFS,Gog:Lic}
(graded inclusion)
which generalizes the ordinary set inclusion. For any $A,B \in L^Y$, we define
a degree $S(A,B)$ of subsethood of $A$ in $B$ by
\begin{align}
  S(A,B) = \textstyle{\bigwedge}_{y \in Y}\bigl(A(y) \rightarrow B(y)\bigr).
  \label{eqn:S}
\end{align}
Clearly, $S(A,B)$ is a general degree in $L$. If $S(A,B) = 1$, we denote the fact
by $A \subseteq B$ and say that $A$ is (fully) included in $B$. Notice that in this
case, we have $A(y) \leq B(y)$ for all $y \in Y$ (this is owing to the fact that
$a \rightarrow b = 1$ if{}f $a \leq b$).

Now, graded attribute implications and their interpretation~\cite{BeVy:ADfDwG}
may be introduced as follows.
Let $Y$ be a finite non-empty set of (symbolic names of) attributes.
A~graded attribute implication in $Y$ is an expression $A \Rightarrow B$
where $A,B \in L^Y$; $A$ is called the antecedent of $A \Rightarrow B$,
$B$ is called the consequent of $A \Rightarrow B$.
Alternatively, graded attribute implications are called 
fuzzy attribute implication~\cite{BeVy:ICFCA} and for brevity we refer to
the formulas as FAIs.
For $A,B,M \in L^Y$, we define the degree 
$||A \Rightarrow B||_M$ to which $A \Rightarrow B$ is true in $M$ by
\begin{align}
  ||A \Rightarrow B||_M &= S(A,M)^* \rightarrow S(B,M).
  \label{eqn:fai_truth}
\end{align}
Recall that in~\eqref{eqn:fai_truth}, $S$ stands for graded subsethood~\eqref{eqn:S},
and ${}^*$ is a hedge on $\mathbf{L}$. Furthermore,
if $\mathcal{M} \subseteq L^Y$, then we put 
\begin{align}
  ||A \Rightarrow B||_\mathcal{M} &=
  \textstyle \bigwedge_{M \in \mathcal{M}}||A \Rightarrow B||_M
  \label{eqn:fai_truth_syst}
\end{align}
and call $||A \Rightarrow B||_\mathcal{M}$ the degree to which 
$A \Rightarrow B$ is true in $\mathcal{M}$.

\begin{remark}
 According to its definition, the degree to which $A \Rightarrow B$
 is true in $M$ depends not only on the operations in $\mathbf{L}$ (namely,
 $\bigwedge$ and $\rightarrow$) but also on the hedge ${}^*$.
 The hedge in~\eqref{eqn:fai_truth} may be seen as a parameter of the interpretation
 of $A \Rightarrow B$ in $M$ and will play an important role in our paper,
 see also~\cite{BeVy:ADfDwG} for detailed comments on the role of hedges.
 Also note that if $M$ is regarded as an $\mathbf{L}$-set of attributes 
 of an object (i.e., $M(y)$ is a degree to which an object has the attribute $y$),
 then $||A \Rightarrow B||_M$ is a degree to which it is true that ``If the object
 has all the attributes from $A$, then it has all the attributes from $B$''.
 This naturally generalizes the ordinary attribute implications and their
 semantics, see~\cite{GaWi:FCA}.
\end{remark}

Consider fixed $Y$ and let $\Sigma$ be a set of FAIs.
An $\mathbf{L}$-set $M \in L^Y$ is called a model of $\Sigma$ whenever
$||A \Rightarrow B||_M = 1$ for all $A \Rightarrow B \in \Sigma$.
The set of all models of $\Sigma$ is denoted by $\mathrm{Mod}(\Sigma)$.
The degree $||A \Rightarrow B||_\Sigma$ to
which $A \Rightarrow B$ is semantically entailed by $\Sigma$ is defined by
\begin{align}
  ||A \Rightarrow B||_\Sigma &=
  \textstyle\bigwedge_{M \in \mathrm{Mod}(\Sigma)}||A \Rightarrow B||_M.
  \label{eqn:sement}
\end{align}
Therefore, the degree to which a FAI (semantically) follows
from $\Sigma$ is defined as the infimum of all degrees to which it is true in all
models of $\Sigma$. This is consistent with the abstract logic framework proposed by
Pavelka~\cite{Pav:Ofl1,Pav:Ofl2,Pav:Ofl3} which was inspired by the influential
paper~\cite{Gog:Lic} by Goguen. Note that since
$\mathrm{Mod}(\Sigma) \subseteq L^Y$,
we may write
$||A \Rightarrow B||_\Sigma = ||A \Rightarrow B||_{\mathrm{Mod}(\Sigma)}$
on account of~\eqref{eqn:fai_truth_syst}. 

In our paper, we exploit a characterization of the semantic entailment which is
based on least models. The system $\mathrm{Mod}(\Sigma)$ of all models of any $\Sigma$
is known to form a particular closure system
(called an $\mathbf{L}^{\!*}$-closure system~\cite{BeFuVy:Fcots}),
see~\cite{BeVy:Pmfai}. Therefore, we may introduce
the least model $[M]_\Sigma$ of $\Sigma$ which contains $M$:
\begin{align}
  [M]_\Sigma
  &= 
  \textstyle\bigcap\{N \in \mathrm{Mod}(\Sigma);\, M \subseteq N\}.
  \label{eqn:least_mod}
\end{align}
The following proposition establishes a characterization of the semantic entailment
by least models and graded subsethood, see~\cite[Theorem 3.11]{BeVy:ADfDwG}.

\begin{proposition}
  For any $\Sigma$ and $A,B \in L^Y$\!:
  $||A \Rightarrow B||_\Sigma = S(B,[A]_\Sigma)$.
  \label{pr:sementail_char}
  \qed
\end{proposition}

In particular, Proposition~\ref{pr:sementail_char} yields that 
$||A \Rightarrow B||_\Sigma = 1$ if{}f $S(B,[A]_\Sigma) = 1$ which is
true if{}f $B$ is fully contained in $[A]_\Sigma$, i.e., $B \subseteq [A]_\Sigma$.
As a further consequence, for given $A$,
$[A]_\Sigma$ is the greatest $\mathbf{L}$-set among all $B \in L^Y$
such that $||A \Rightarrow B||_\Sigma = 1$.

Let $\Sigma$ and $\Gamma$ be sets of FAIs in $Y$.
We call $\Sigma$ and $\Gamma$ equivalent whenever
$||A \Rightarrow B||_\Sigma = ||A \Rightarrow B||_\Gamma$
for all $A,B \in L^Y$. In words, $\Sigma$ and $\Gamma$ are equivalent
whenever they entail each FAI to the same degree.
The following proposition shows that the condition can be restated in
several equivalent ways, see~\cite{BeVy:ICFCA,BeVy:ADfDwG}.

\begin{proposition}\label{pr:equivalence}
  For any $\Sigma$ and $\Gamma$, the following conditions are equivalent:
  \begin{enumerate}\parskip=0pt%
  \item[\itm{1}]
    $\Sigma$ and $\Gamma$ are equivalent,
  \item[\itm{2}]
    for all $A,B \in L^Y\!\!:$
    $||A \Rightarrow B||_\Sigma = 1$ if{}f $||A \Rightarrow B||_\Gamma = 1$,
  \item[\itm{3}]
    for all $A \Rightarrow B \in \Sigma\!:$ $||A \Rightarrow B||_\Gamma = 1$ and
    \newline
    for all $C \Rightarrow D \in \Gamma\!:$ $||C \Rightarrow D||_\Sigma = 1$,
  \item[\itm{4}]
    $\mathrm{Mod}(\Sigma) = \mathrm{Mod}(\Gamma)$.
    \qed
  \end{enumerate}
\end{proposition}

We conclude the preliminaries by the following remark on alternative semantics
and axiomatizations of the semantic entailment of FAIs.

\begin{remark}
  (a)
  The notion of a degree of semantic entailment used in this paper is defined
  in a way which generalizes the classic propositional semantics of attribute
  implications. That is, for $M \in \mathrm{Mod}(\Sigma)$, the degree $M(y)$
  is interpreted as the degree to which $y$ is present in $M$. Thus, if
  attributes are considered as propositional variables, $M$ may be seen
  as their evaluation prescribing degrees to which the propositional
  variables are true. Since the classic attribute implications have an
  alternative database semantics~\cite{DeCa,Fagin,SaDePaFa:Ebrddfpl}
  which yields the same notion of semantic entailment, it may be tempting
  to look for an analogous alternative semantics in the graded setting.
  In~\cite{BeVy:DASFAA}, it is shown that such an alternative semantics
  exists and that FAIs may alternatively be
  seen as formulas prescribing similarity-based dependencies in
  relational databases~\cite{Mai:TRD}.

  (b)
  The semantic entailment introduced in this paper has (several) interesting
  Armstrong-style~\cite{Arm:Dsdbr} axiomatizations. An inference system which is
  complete over arbitrary $\mathbf{L}$ is presented in \cite{BeVy:Falcrl}.
  Note that the inference system presented therein contains an infinitary
  rule which may be disregarded in some important cases~\cite{BeVy:ADfDwG,Vy:Rfal},
  cf. also~\cite{KuVy:Flprrai} for an alternative axiomatization
  with an infinitary rule.
  An alternative inference systems based on the rules of simplification in
  presented in~\cite{BeCoEnMoVy:Aermdsod}. A graph-based inference system for
  FAIs is presented in~\cite{UrVy:Dddosbd}.
\end{remark}

\section{Non-Redundant Bases: Overview and Related Work}\label{sec:bases_survey}
In this section, we present an overview of existing results regarding bases of
FAIs. The existing approaches are concerned with
describing bases of object-attribute data with graded attributes which are
formalized as formal $\mathbf{L}$-contexts: For non-empty finite sets $X$
(set of objects) and $Y$ (set of attributes), and a binary $\mathbf{L}$-relation
$I\!: X \times Y \to L$, the triplet $\mathbf{I} = \langle X,Y,I\rangle$ is
called a formal $\mathbf{L}$-context~\cite{Bel:FRS}.
A formal $\mathbf{L}$-context $\mathbf{I}$
induces a couple of operators ${}^{\uparrow}\!: L^X \to L^Y$ and 
${}^{\downarrow}\!: L^Y \to L^X$ defined by
\begin{align}
  A^{\uparrow}(y) &=
  \textstyle\bigwedge_{x \in X}\bigl(A(x)^* \rightarrow I(x,y)\bigr),
  \label{eqn:up} \\
  B^{\downarrow}(x) &=
  \textstyle\bigwedge_{y \in Y}\bigl(B(y) \rightarrow I(x,y)\bigr),
  \label{eqn:dn}
\end{align}
for all $A \in L^X$, $B \in L^Y$, $x \in X$, and $y \in Y$.
The operators ${}^{\downarrow},{}^{\uparrow}$ form
a so-called Galois connection with hedge~\cite{BeVy:Fcalh} and
their composition ${}^{\downarrow\uparrow}$ is
an $\mathbf{L}^{\!*}$-closure operator~\cite{BeFuVy:Fcots}.
Given $\mathbf{I} = \langle X,Y,I\rangle$, which represents input data,
we define the degree $||A \Rightarrow B||_\mathbf{I}$
to which $A \Rightarrow B$ ($A,B \in L^Y$) is true
in $\mathbf{I}$, see~\cite{BeChVy:Ifdwfa}, as follows:
\begin{align}
  ||A \Rightarrow B||_\mathbf{I} &=
  \textstyle
  \bigwedge_{x \in X}||A \Rightarrow B||_{\{{}^{1\!}/x\}^{\uparrow}}.
\end{align}
Hence, $||A \Rightarrow B||_\mathbf{I}$ may be understood as a generalization
of the ordinary notion of an attribute implication valid in a formal context:
$||A \Rightarrow B||_\mathbf{I}$ is the degree to which the following
condition is true: ``For each object $x \in X$, if the object has all the
attributes from $A$, then it has all the attributes from $B$''.

Now, the basic problem regarding FAIs and formal $\mathbf{L}$-contexts
is the following:
Given $\mathbf{I} = \langle X,Y,I\rangle$, find $\Sigma$ such that 
\begin{align}
  ||A \Rightarrow B||_\Sigma &= ||A \Rightarrow B||_\mathbf{I}
\end{align}
for all $A,B \in L^Y$. Such a $\Sigma$ is called complete in $\mathbf{I}$.
In addition, if $\Sigma$ is non-redundant (or minimal), then it is called
a non-redundant (or minimal) base of~$\mathbf{I}$. The notion of non-redundancy
is considered the usual way: $A \Rightarrow B \in \Sigma$ is redundant
in $\Sigma$ whenever 
$||A \Rightarrow B||_{\Sigma \setminus \{A \Rightarrow B\}} = 1$;
$\Sigma$ is non-redundant whenever there is no $A \Rightarrow B \in \Sigma$
which is redundant in $\Sigma$. Analogously, $\Sigma$ being minimal
means that there is no $\Gamma$ which is equivalent to $\Sigma$ such
that $|\Gamma| < |\Sigma|$.

The investigation of complete sets and bases in the graded setting started
with~\cite{Po:FB} where the author generalizes the ordinary notion of
a pseudo-intent~\cite{GuDu}, see also~\cite{Ga:Tbaca,GaWi:FCA}. In this
setting, the hedges were not involved as parameters of the semantics of
FAIs as in~\eqref{eqn:fai_truth}
which may be viewed in our general setting so that ${}^*$ is taken
as the identity. In~\cite{Po:FB}, $P \in L^Y$ is called
a pseudo-intent (of $\mathbf{I}$) whenever
$P \ne P^{\downarrow\uparrow}$ (i.e., $P \subset P^{\downarrow\uparrow}$) and 
\begin{align}
  \text{for each pseudo-intent }
  Q \subset P
  \text{, we have }
  Q^{\downarrow\uparrow} \subseteq P.
  \label{eqn:PoP}
\end{align}
Therefore, the definition of the notion of a pseudo-intent copies the
classic definition except for ${}^{\downarrow\uparrow}$ is given
by~\eqref{eqn:up} and~\eqref{eqn:dn}. If $\mathbf{L}$ is finite,
pseudo-intents exist and are uniquely given (recall that in our paper,
we consider $Y$ always finite). Furthermore, \cite{Po:FB}
observes that
\begin{align}
  \Sigma &=
  \{P \Rightarrow P^{\downarrow\uparrow};\, P \text{ is a pseudo-intent}\}
\end{align}
is complete in $\mathbf{I}$ but in general, $\Sigma$ is redundant.
Clearly, for $\mathbf{L}$ being the two-element Boolean algebra,
the notion of a pseudo-intent coincides with the classic one~\cite{GuDu}.

In~\cite{BeChVy:Ifdwfa}, the authors propose a different notion of
pseudo-intents in the graded setting. Namely, the approach in~\cite{BeChVy:Ifdwfa}
started with using general hedges as parameters of the interpretation of FAIs
in formal $\mathbf{L}$-contexts (similar approach to parameterizations of
if-then rules appeared in~\cite{BeVy:FHLI,BeVy:FHLII}).
In addition to that, the paper introduces
a general concept of a system of pseudo-intents: Put
\begin{align}
  \mathcal{U} = \{P \in L^Y;\, P \ne P^{\downarrow\uparrow}\}
  \label{eqn:U}
\end{align}
and call $\mathcal{P} \subseteq \mathcal{U}$ a system of pseudo-intents
whenever for each $P \in \mathcal{U}$, we have
\begin{align}
  P \in \mathcal{P}
  \text{ if{}f }
  ||Q \Rightarrow Q^{\downarrow\uparrow}||_P = 1 
  \text{ for any }
  Q \in \mathcal{P}
  \text{ such that }
  Q \ne P.
  \label{eqn:system_P}
\end{align}
The results in~\cite{BeChVy:Ifdwfa,BeVy:Falaitvenb} show that if ${}^*$
is globalization, then provided that $\mathbf{L}$ is finite, $\mathbf{I}$ admits
a unique system of pseudo-intents which determines a minimal base
\begin{align}
  \Sigma &=
  \{P \Rightarrow P^{\downarrow\uparrow};\, P \in \mathcal{P}\}
  \label{eqn:nred}
\end{align}
of $\mathbf{I}$ analogously as in the classic case. In fact, for ${}^*$
being the globalization, \eqref{eqn:system_P} translates into~\eqref{eqn:PoP}.
In~\cite{Vy:Omsgai}, a criterion for minimality of a general set of FAIs
for ${}^*$ being the globalization is described.

The analysis in~\cite{BeVy:Falaitvenb} further showed that for general hedges,
the systems of pseudo-intents are not given uniquely and may have different
sizes and, in case of infinite $\mathbf{L}$, may not even exist,
cf.~\cite[Example 5.13]{BeVy:ADfDwG}. On the other hand, if there is a system
$\mathcal{P}$ of pseudo-intents of $\mathbf{I}$, then~\eqref{eqn:nred}
always determines a non-redundant base.

In order to compute general systems of pseudo-intents considering general hedges,
a graph-based method has been announced in~\cite{BeVy:Faicnbumis} and further
described in~\cite{BeVy:Cnbirdtga}. The method is based on an observation
that systems of pseudo-intents coincide with particular maximal independent
sets in graphs induced by $\mathbf{I}$. Namely, we can introduce
\begin{align}
  E = \{\langle P,Q\rangle\! \in \mathcal{U} \times \mathcal{U};\,
  P \ne Q \text{ and } ||Q \Rightarrow Q^{\downarrow\uparrow}||_P \ne 1\}
\end{align}
If $\mathcal{U}$ (defined as before) is non-empty,
then $\mathbf{G} = \langle \mathcal{U},E \cup E^{-1}\rangle$ is
a graph. Furthermore, for any $\mathcal{P} \subseteq \mathcal{U}$,
\cite{BeVy:Cnbirdtga} defines the following subsets of $\mathcal{U}$:
\begin{align}
  \mathrm{Pred}(\mathcal{P}) &=
  \textstyle\bigcup_{Q \in \mathcal{P}}\{P \in \mathcal{U};\,
  \langle P,Q\rangle \in E\}.
\end{align}
The main observations of~\cite{BeVy:Cnbirdtga} which allow
to determine systems of pseudo-intents as particular maximal independent
sets are the following:
\begin{enumerate}\parskip=0pt%
\item[\itm{1}]
  $\mathcal{P}$ is a system of pseudo-intents if{}f
  $\mathcal{U} \setminus \mathcal{P} = \mathrm{Pred}(\mathcal{P})$;
\item[\itm{2}]
  If $\mathcal{U} \setminus \mathcal{P} = \mathrm{Pred}(\mathcal{P})$,
  then $\mathcal{P}$ is a maximal independent set in $\mathbf{G}$.
\end{enumerate}
The implications of~\cite{BeVy:Cnbirdtga} are more or less just theoretical
because in practice one is unable to use such a graph-based procedure to find
a system of pseudo-intents---because of the enormous size of $\mathbf{G}$,
enumerating of all maximal independent sets satisfying the additional condition
\itm{1} is intractable. Furthermore, the description does not
answer the question if for any finite $\mathbf{L}$ and arbitrary hedge
${}^*$ there exists at least one system of pseudo-intents.
This remains an open problem~\cite{Kw:Open2006}.

\section{Results}\label{sec:results}
The first observation we present in this section involves sets of FAIs
in a special form. From the model-theoretic point of
view, we show that for each set of FAIs, one can
find an equivalent set where the consequents of all the FAIs
contained in the set are models. This property is introduced
in the following definition.

\begin{definition}
  Let $\Sigma$ be a set of FAIs. We say that the FAIs in $\Sigma$ have
  \emph{saturated consequents} whenever for every $A \Rightarrow B \in \Sigma$,
  we have $[A]_\Sigma \subseteq B$.
\end{definition}

Obviously, whether a given $A \Rightarrow B$ has a saturated consequent
depends on $\Sigma$ from which it is taken.
Applying Proposition~\ref{pr:sementail_char}, it follows that
FAIs in $\Sigma$ have saturated consequents if{}f
for every $A \Rightarrow B \in \Sigma$, we have $B = [A]_\Sigma$.
Therefore, if FAIs in $\Sigma$ have saturated consequents,
then all FAIs in $\Sigma$ are of the form $A \Rightarrow [A]_\Sigma$.
The following assertion shows that each set of FAIs admits
an equivalent set of FAIs with saturated consequents.

\begin{lemma}\label{le:saturated}
  Let $\Gamma$ be a set of FAIs and let
  \begin{align}
    \Sigma = \{A \Rightarrow [A]_\Gamma;\, A \Rightarrow B \in \Gamma\}.
    \label{eqn:infl_cons}
  \end{align}
  Then, $\Sigma$ and $\Gamma$ are equivalent.
  In addition, if $\Gamma$ is minimal then so is $\Sigma$.
\end{lemma}
\begin{proof}
  In order to prove the first part of the claim,
  according to Proposition~\ref{pr:equivalence} it suffices to check that $\Gamma$
  and $\Sigma$ given by~\eqref{eqn:infl_cons} have the same models. This can be
  checked using Proposition~\ref{pr:sementail_char} as follows.

  Let $M \in \mathrm{Mod}(\Gamma)$.
  In order to prove that $M \in \mathrm{Mod}(\Sigma)$, it suffices to prove
  $||A \Rightarrow [A]_\Gamma||_M = 1$ for each
  $A \Rightarrow [A]_\Gamma \in \Sigma$ which is indeed true:
  Using Proposition~\ref{pr:sementail_char},
  we get $||A \Rightarrow [A]_\Gamma||_\Gamma = 1$ on the account of
  $[A]_\Gamma \subseteq [A]_\Gamma$. Since $M \in \mathrm{Mod}(\Gamma)$,
  we therefore have $||A \Rightarrow [A]_\Gamma||_M = 1$.

  Conversely, let $M \in \mathrm{Mod}(\Sigma)$ and take any
  $A \Rightarrow B \in \Gamma$.
  Observe that $||A \Rightarrow B||_\Gamma = 1$ and thus $B \subseteq [A]_\Gamma$
  owing to Proposition~\ref{pr:sementail_char}.
  Since $A \Rightarrow [A]_\Gamma \in \Sigma$, it follows that
  $S(A,M)^* \leq S([A]_\Gamma,M)$ which further gives
  \begin{align*}
    S(A,M)^* \leq S([A]_\Gamma,M) \leq S(B,M)
  \end{align*}
  because $B \subseteq [A]_\Gamma$ and the graded subsethood is antitone in
  the first argument, i.e., $S(B_1,M) \leq S(B_2,M)$ whenever $B_2 \subseteq B_1$.
  Hence, $S(A,M)^* \leq S(B,M)$ gives $||A \Rightarrow B||_M = 1$.

  The second claim is an easy consequence of the first one: Suppose that $\Gamma$
  is minimal. Since $\Sigma$ and $\Gamma$ are equivalent, we then have
  $|\Gamma| \leq |\Sigma|$. Directly from~\eqref{eqn:infl_cons}, it follows
  that $|\Sigma| \leq |\Gamma|$ and thus $|\Gamma| = |\Sigma|$ which shows
  that $\Sigma$ is minimal as well.
\end{proof}

Since $\Gamma$ and $\Sigma$ given by~\eqref{eqn:infl_cons} are equivalent,
it is easy to see that each FAI in $\Sigma$ is of the form
$A \Rightarrow [A]_\Sigma$, i.e., the consequents on FAIs in $\Sigma$
are saturated. Also note that if the FAIs in $\Gamma$ already have
saturated consequents, then $\Sigma = \Gamma$ for $\Sigma$ given
by~\eqref{eqn:infl_cons}.

\begin{example}
  (a)
  Let us note that $\Sigma$ given by~\eqref{eqn:infl_cons} may be
  strictly smaller than $\Gamma$ in terms of the number of formulas.
  This is true even in the case when $\mathbf{L}$ is
  the two-element Boolean algebra. Indeed, for
  $\Gamma$ given by
  \begin{align*}
    \Gamma = \{\{p\} \Rightarrow \{q\}, \{p\} \Rightarrow \{r\}\},
  \end{align*}
  we obviously have $[\{p\}]_\Gamma = \{q,r\}$ and thus the corresponding
  $\Sigma$ given by \eqref{eqn:infl_cons} is of the form
  \begin{align*}
    \Sigma = \{\{p\} \Rightarrow \{q,r\}\}.
  \end{align*}

  (b)
  Notice that in the previous case, both $\Sigma$ and $\Gamma$ are
  non-redundant. In general, given a non-redundant $\Gamma$, it may
  happen that the corresponding $\Sigma$ given by \eqref{eqn:infl_cons}
  is redundant. For instance, consider $\Gamma$ as follows:
  \begin{align*}
    \Gamma = \{\{\} \Rightarrow \{p\}, \{p\} \Rightarrow \{q\}\}.
  \end{align*}
  Since $[\{\}]_\Gamma = [\{p\}]_\Gamma = \{p,q\}$, $\Sigma$
  given by \eqref{eqn:infl_cons} is of the form
  \begin{align*}
    \Sigma = \{\{\} \Rightarrow \{p,q\}, \{p\} \Rightarrow \{p,q\}\}.
  \end{align*}
  Now, observe that $\Gamma$ is non-redundant while $\Sigma$ is
  redundant because $\{p\} \Rightarrow \{p,q\}$ is redundant in $\Sigma$.
  Therefore, unlike in the case of minimality, see Lemma~\ref{le:saturated},
  non-redundancy is not preserved by saturating the consequents of
  FAIs as in~\eqref{eqn:infl_cons}.
\end{example}

Lemma~\ref{le:saturated} allows us to restrict our considerations on bases
only to sets of FAIs with saturated consequents. Thus, for brevity, for
any $\Gamma$ consisting of FAIs over~$Y$, we let
\begin{align}
  \univ &=
  \{A \Rightarrow [A]_\Gamma;\, A \in L^Y \text{ and } A \ne [A]_\Gamma\}.
\end{align}
Trivially, $\Gamma$ and $\univ$ are equivalent and the consequents in all FAIs
in $\univ$ are obviously saturated. For subsets $\Sigma \subseteq \univ$,
we have the following if-and-only-if condition for $\Sigma$ and $\Gamma$
being equivalent.

\begin{theorem}\label{th:compl}
  For any $\Sigma \subseteq \univ$, the following conditions are equivalent:
  \begin{itemize}\parskip=0pt%
  \item[\itm{1}]
    $\Sigma$ and $\Gamma$ are equivalent as sets of FAIs.
  \item[\itm{2}]
    For every $A \in L^{Y\!}$ such that $A \ne [A]_\Gamma\kern-2pt:$
    \newline
    if $A \in \mathrm{Mod}(\Sigma \setminus \{A \Rightarrow [A]_\Gamma\})$,
    then $A \Rightarrow [A]_\Gamma \in \Sigma$.
  \end{itemize}
\end{theorem}
\begin{proof}
  Let $\Sigma$ and $\Gamma$ be equivalent. Take any $A \in L^{Y\!}$
  such that $A \ne [A]_\Gamma$. Moreover, let us assume that
  $A \Rightarrow [A]_\Gamma \not\in \Sigma$. Obviously,
  $\Sigma \setminus \{A \Rightarrow [A]_\Gamma\} = \Sigma$,
  i.e., in order to prove \itm{2}, it suffices to show that
  \begin{align*}
    A \not\in
    \mathrm{Mod}(\Sigma \setminus \{A \Rightarrow [A]_\Gamma\}) =
    \mathrm{Mod}(\Sigma)
  \end{align*}
  but this is indeed the case:
  $A \ne [A]_\Gamma$ means $A \not\in \mathrm{Mod}(\Gamma)$
  and so Proposition~\ref{pr:equivalence} gives 
  $A \not\in \mathrm{Mod}(\Sigma)$ because $\Sigma$ and $\Gamma$
  are equivalent.

  Conversely, assume that \itm{2} is satisfied. Since $\Sigma \subseteq \univ$
  and in addition $\Gamma$ and $\univ$ are equivalent,
  Proposition~\ref{pr:equivalence} yields
  \begin{align*}
    \mathrm{Mod}(\Gamma) =
    \mathrm{Mod}(\univ)
    \subseteq
    \mathrm{Mod}(\Sigma).
  \end{align*}
  Therefore, in order to prove the equivalence of $\Gamma$ and $\Sigma$,
  it suffices to check the converse inclusion.
  By contradiction, let us assume that
  $\mathrm{Mod}(\Sigma) \nsubseteq \mathrm{Mod}(\Gamma)$.
  Using this assumption,
  there is $A \in \mathrm{Mod}(\Sigma)$ such that $A \not\in \mathrm{Mod}(\Gamma)$.
  Utilizing the fact $A \not\in \mathrm{Mod}(\Gamma)$,
  it directly follows that $A \ne [A]_\Gamma$.
  In addition, $A \in \mathrm{Mod}(\Sigma)$ gives 
  $A \in \mathrm{Mod}(\Sigma \setminus \{A \Rightarrow [A]_\Gamma\})$.
  Hence, using \itm{2}, we get $A \Rightarrow [A]_\Gamma \in \Sigma$.
  Using the fact $A \in \mathrm{Mod}(\Sigma)$ again,
  $||A \Rightarrow [A]_\Gamma||_A = 1$, i.e.,
  $S(A,A)^* \leq S([A]_\Gamma, A)$. The last inequality gives
  \begin{align*}
    1 = 1^* = S(A,A)^* \leq S([A]_\Gamma, A),
  \end{align*}
  which in fact shows that $[A]_\Gamma \subseteq A$ which together with
  the extensivity $A \subseteq [A]_\Gamma$ of $[{\cdots}]_\Gamma$
  contradicts $A \ne [A]_\Gamma$. Therefore, 
  $\mathrm{Mod}(\Sigma) \subseteq \mathrm{Mod}(\Gamma)$, i.e., $\Sigma$ and $\Gamma$
  are equivalent owing to Proposition~\ref{pr:equivalence}.
\end{proof}

The following lemma shows that the converse implication to that in
Theorem~\ref{th:compl}\,\itm{2} constitutes a sufficient condition of
non-redundancy.

\begin{lemma}\label{le:witred}
  Let $\Sigma \subseteq \univ$ be a set of FAIs satisfying
  the following condition:
  For every $A \Rightarrow [A]_\Gamma \in \Sigma$, we have
  $A \in \mathrm{Mod}(\Sigma \setminus \{A \Rightarrow [A]_\Gamma\})$.
  Then, $\Sigma$ is non-redundant.
\end{lemma}
\begin{proof}
  Clearly, for any $A \Rightarrow [A]_\Gamma \in \Sigma$,
  using the fact that $A \Rightarrow [A]_\Gamma \in \univ$, i.e.,
  $A \ne [A]_\Gamma$, it follows that $||A \Rightarrow [A]_\Gamma||_A < 1$
  and thus $A \not\in \mathrm{Mod}(\Sigma)$. Hence,
  $A \in \mathrm{Mod}(\Sigma \setminus \{A \Rightarrow [A]_\Gamma\})$
  gives that $A \Rightarrow [A]_\Gamma$ is not redundant in $\Sigma$ because
  the theories $\Sigma \setminus \{A \Rightarrow [A]_\Gamma\}$ and 
  $\Sigma$ are not equivalent.
\end{proof}

\begin{example}\label{ex:witred}
  The converse implication to that in Lemma~\ref{le:witred} does
  not hold in general. For instance, let $\mathbf{L}$ be
  a three-element \L ukasiewicz chain and let ${}^*$ be
  the identity on $L = \{0,0.5,1\}$. Furthermore, consider
  $\Sigma$ such as
  \begin{align*}
    \Sigma = \{\{p\} \Rightarrow \{p,q\}, \{\} \Rightarrow \{{}^{0.5\!}/q\}\}.
  \end{align*}
  Observe that $\Sigma$ is non-redundant.
  Indeed, we have that $\{\} \not\in \mathrm{Mod}(\Sigma)$
  because $||\{\} \Rightarrow \{{}^{0.5\!}/q\}||_{\{\}} = 0.5 < 1$
  and $||\{p\} \Rightarrow \{p,q\} ||_{\{\}} = 1$, i.e.,
  $\{\} \Rightarrow \{{}^{0.5\!}/q\}$ is not redundant in $\Sigma$.
  Furthermore,
  \begin{align*}
    ||\{p\} \Rightarrow \{p,q\}||_{\{p,{}^{0.5\!}/q\}} &=
    1 \rightarrow S(\{p,q\},\{p,{}^{0.5\!}/q\}) = 1 \rightarrow 0.5 = 0.5 < 1,
  \end{align*}
  i.e., $\{p,{}^{0.5\!}/q\} \not\in \mathrm{Mod}(\Sigma)$. On the other hand,
  we clearly have
  \begin{align*}
    ||\{\} \Rightarrow \{{}^{0.5\!}/q\}||_{\{p,{}^{0.5\!}/q\}} &=
    1 \rightarrow S(\{{}^{0.5\!}/q\},\{p,{}^{0.5\!}/q\}) = 1 \rightarrow 1 = 1.
  \end{align*}
  Altogether, $\{p\} \Rightarrow \{p,q\}$ is not redundant in $\Sigma$.
  Also note that both FAIs in $\Sigma$ have saturated consequents since
  $[\{p\}]_\Sigma = \{p,q\}$ and $[\{\}]_\Sigma = \{{}^{0.5\!}/q\}$.
  Now, observe that for $\{p\} \Rightarrow \{p,q\} \in \Sigma$, we have
  \begin{align*}
    ||\{\} \Rightarrow \{{}^{0.5\!}/q\}||_{\{p\}} &=
    1 \rightarrow S(\{{}^{0.5\!}/q\},\{p\}) \\
    &= 1 \rightarrow ((0 \rightarrow 1) \wedge (0.5 \rightarrow 0)) \\
    &= 1 \rightarrow (1 \wedge 0.5) = 1 \rightarrow 0.5 = 0.5 < 1,
  \end{align*}
  which shows that
  $\{p\} \not\in \mathrm{Mod}(\Sigma \setminus \{\{p\} \Rightarrow \{p,q\}\})$,
  i.e., the converse implication to that in Lemma~\ref{le:witred} does not
  hold for general $\mathbf{L}$. Let us note that an analogous observation can
  also be made for ${}^*$ being~\eqref{eqn:glob} or for $\mathbf{L}$ being
  the two-element Boolean algebra in which case at least three distinct attributes
  must be used to find a counterexample.
\end{example}

In the rest of this section, we pay attention to a condition which is derived
from the assumption in Lemma~\ref{le:witred}. As we have shown, the condition
in Lemma~\ref{le:witred} is \emph{sufficient} for non-redundancy but
not \emph{necessary} in general.
Indeed, Example~\ref{ex:witred} shows a non-redundant $\Sigma \subseteq \univ$
and a particular $A \Rightarrow [A]_\Gamma \in \Sigma$ such that 
$A \not\in \mathrm{Mod}(\Sigma \setminus \{A \Rightarrow [A]_\Gamma\})$.
In general, if 
$A \in \mathrm{Mod}(\Sigma \setminus \{A \Rightarrow [A]_\Gamma\})$
for $A \Rightarrow [A]_\Gamma \in \Sigma$, we may say that $A$
acts as a ``witness'' of
the non-redundancy of $A \Rightarrow [A]_\Gamma$ in $\Sigma$
because $A$ is a model of $\Sigma \setminus \{A \Rightarrow [A]_\Gamma\}$ but
it is not a model of $\Sigma$, see Proposition~\ref{pr:equivalence}.
If $A \not\in \mathrm{Mod}(\Sigma \setminus \{A \Rightarrow [A]_\Gamma\})$,
$A$ does not act as such a witness because $A \not\in \mathrm{Mod}(\Sigma)$ and
$A \not\in \mathrm{Mod}(\Sigma \setminus \{A \Rightarrow [A]_\Gamma\})$.
Therefore, for a non-redundant $\Sigma \subseteq \univ$ we may consider whether
its non-redundancy is witnessed by antecedents of FAIs in $\Sigma$ which is,
in fact, a stronger requirement on non-redundancy. We introduce the key
notion as follows.

\begin{definition}
  Let $\Sigma$ be a non-redundant set of FAIs. We say that the
  non-redundancy of $\Sigma$ is \emph{witnessed}
  (by the antecedents of FAIs in $\Sigma$)
  whenever for every $A \Rightarrow B \in \Sigma$, we have
  that $A \in \mathrm{Mod}(\Sigma \setminus \{A \Rightarrow B\})$.
\end{definition}

\begin{remark}
  Let us note that the general property
  $A \in \mathrm{Mod}(\Sigma \setminus \{A \Rightarrow B\})$ whenever
  $A \Rightarrow B \in \Sigma$ may also be possessed by sets of FAIs
  which are redundant. For instance,
  one can consider $\Sigma = \{\{\} \Rightarrow \{\}\}$ which trivially
  has this property.
\end{remark}

In order to prove the existence of sets of FAIs which witnessed
non-redundancy for a particular setting of structures of degrees and hedges,
we utilize the following technical observation which ensures the existence
of a total strict order on antecedents of FAIs in a given set of FAIs.

\begin{lemma}\label{le:lexlt}
  Let $\mathbf{L}$ be a complete residuated lattice with globalization,
  $\Gamma$ be a finite non-redundant set of FAIs with saturated consequents,
  and
  \begin{align}
    \mathcal{X} =
    \bigl\{A \in L^{\!Y};\, A \Rightarrow [A]_\Gamma \in \Gamma\bigr\}.
  \end{align}
  Then, there exists a strict total order relation $\lexlt$ on $\mathcal{X}$
  such that 
  \begin{align}
    [B]_{\Gamma \setminus \{B \Rightarrow [B]_\Gamma\}} =
    [B]_{\{A \Rightarrow [A]_\Gamma \in \Gamma;\, A \lexlt B\}}
    \label{eqn:lexlt}
  \end{align}
  for all $B \in \mathcal{X}$.
\end{lemma}
\begin{proof}
  It is immediate that the claim holds trivially for $\Gamma = \emptyset$.
  We inspect the situation for $\Gamma \ne\emptyset$.
  The finiteness of $\Gamma$ gives that $\mathcal{X}$ is finite as well.
  We proceed
  by induction and assume that we have already found
  $A_1,\ldots,A_n \in \mathcal{X}$
  ($n \geq 0$) such that $A_1 \lexlt \cdots \lexlt A_n$ and we put
  \begin{align*}
    \Psi = \{A_1 \Rightarrow [A_1]_\Gamma,\ldots,A_n \Rightarrow [A_n]_\Gamma\}.
  \end{align*}
  In order to prove the assertion, we check that if $\Gamma \setminus \Psi$
  is non-empty, then it contains some $B \Rightarrow [B]_\Gamma$ such that
  \begin{align*}
    [B]_{\Gamma \setminus \{B \Rightarrow [B]_\Gamma\}} = [B]_\Psi
  \end{align*}
  from which we immediately get that $A_1 \lexlt \cdots \lexlt A_n$
  can be extended by $A_n \lexlt B$, see \eqref{eqn:lexlt}.
  By contradiction, let us assume that $\Gamma \setminus \Psi$ is non-empty and
  no such $B \Rightarrow [B]_\Gamma \in \Gamma \setminus \Psi$ exists.
  Thus, we assume that for each $B \Rightarrow [B]_\Gamma \in
  \Gamma \setminus \Psi$ we have 
  $[B]_\Psi \subset [B]_{\Gamma \setminus \{B \Rightarrow [B]_\Gamma\}}$
  because $\Psi \subseteq \Gamma \setminus \{B \Rightarrow [B]_\Gamma\}$.

  Observe that for each
  $B_0 \Rightarrow [B_0]_\Gamma \in \Gamma \setminus \Psi$, there is
  $\Psi_1 \ne\emptyset$ which is minimal in the number of contained
  FAIs such that $\Psi \cup \Psi_1 \subseteq
  \Gamma \setminus \{B_0 \Rightarrow [B_0]_\Gamma\}$ and
  \begin{align*}
    [B_0]_{\Psi \cup \Psi_1} =
    [B_0]_{\Gamma \setminus \{B_0 \Rightarrow [B_0]_\Gamma\}}.
  \end{align*}
  Therefore, using the fact that ${}^*$ is globalization together with the
  last equality and the fact that 
  $[B_0]_{\Psi} \subset
  [B_0]_{\Gamma \setminus \{B_0 \Rightarrow [B_0]_\Gamma\}}$,
  it follows that there must be some
  $B_1 \Rightarrow [B_1]_\Gamma \in \Psi_1 \setminus \Psi$
  with $B_1 \ne B_0$ such that the following conditions are satisfied:
  \begin{itemize}\parskip=0pt
  \item
    $B_1 \subseteq [B_0]_\Psi$,
  \item
    $[B_1]_\Gamma \subseteq
    [B_0]_{\Gamma \setminus \{B_0 \Rightarrow [B_0]_\Gamma\}}$, and thus
  \item
    $[B_1]_\Gamma \subseteq [B_0]_\Gamma$.
  \end{itemize}
  Now, the same observation can be made for $B_1 \Rightarrow [B_1]_\Gamma$.
  That is, there is some $B_2 \Rightarrow [B_2]_\Gamma \in \Gamma \setminus \Psi$
  such that the following conditions are satisfied:
  \begin{itemize}\parskip=0pt
  \item
    $B_2 \subseteq [B_1]_\Psi$,
  \item
    $[B_2]_\Gamma \subseteq
    [B_1]_{\Gamma \setminus \{B_1 \Rightarrow [B_1]_\Gamma\}}$, and
  \item
    $[B_2]_\Gamma \subseteq [B_1]_\Gamma$.
  \end{itemize}
  Therefore, we may repeat the idea over and over again to form
  a sufficiently long sequence $B_0,B_1,B_2,\ldots$ in which, owing to
  the finiteness of $\Gamma$, there must be two indices $i < j$
  such that $B_i = B_j$. As a consequence
  of the above-listed properties of the elements in the sequence, we obtain
  $[B_i]_\Psi = \cdots = [B_j]_\Psi$ as well as
  $[B_i]_\Gamma = \cdots = [B_j]_\Gamma$.
  Now, observe that using $B_i \ne B_{i+1}$, we trivially get
  \begin{align*}
    ||B_{i+1} \Rightarrow [B_{i+1}]_\Gamma||_{\Gamma \setminus
      \{B_i \Rightarrow [B_i]_\Gamma\}} = 1
  \end{align*}
  because $B_{i+1} \Rightarrow [B_{i+1}]_\Gamma \in \Gamma$. Moreover,
  the fact that $[B_i]_\Gamma = [B_{i+1}]_\Gamma$ yields
  \begin{align}
    ||B_{i+1} \Rightarrow [B_i]_\Gamma||_{\Gamma \setminus
      \{B_i \Rightarrow [B_i]_\Gamma\}} = 1.
    \label{eqn:aux_tr1}
  \end{align}
  Finally, using $B_{i+1} \subseteq [B_{i+1}]_\Psi = [B_{i}]_\Psi$
  and Proposition~\ref{pr:sementail_char}, we get that
  \begin{align}
    1 =
    ||B_{i} \Rightarrow B_{i+1}||_\Psi \leq 
    ||B_{i} \Rightarrow B_{i+1}||_{\Gamma \setminus
      \{B_i \Rightarrow [B_i]_\Gamma\}}
    \label{eqn:aux_tr2}
  \end{align}
  on the account of
  $\Psi \subseteq \Gamma \setminus \{B_i \Rightarrow [B_i]_\Gamma\}$.
  Since the semantic entailment of FAIs is transitive,
  from~\eqref{eqn:aux_tr1} and \eqref{eqn:aux_tr2} we further get
  \begin{align*}
    ||B_i \Rightarrow [B_i]_\Gamma||_{\Gamma \setminus
      \{B_i \Rightarrow [B_i]_\Gamma\}} = 1
  \end{align*}
  which contradicts the non-redundancy of $\Gamma$.
  As a result, $A_1 \lexlt \cdots \lexlt A_n$ can
  always be extended by some $B \in \mathcal{X}$
  satisfying~\eqref{eqn:lexlt} provided that
  $\Gamma \setminus \Psi$ is non-empty. Since $\Gamma$
  is finite, this ultimately defines a strict total order
  on $\mathcal{X}$ satisfying~\eqref{eqn:lexlt}.
\end{proof}

\begin{example}\label{ex:nlukid}
  Let us note that in general the existence of the strict total
  order described in Lemma~\ref{le:lexlt} is not ensured if ${}^*$ is
  other than the globalization. For illustration, let us consider
  the same structure of truth degrees as in Example~\ref{ex:witred}
  with ${}^*$ being the identity. Furthermore, consider $\Gamma$
  such as
  \begin{align*}
    \Gamma &=
    \{\{{}^{0.5\!}/r\} \Rightarrow \{p,{}^{0.5\!}/q,{}^{0.5\!}/r\},
    \{\} \Rightarrow \{p\}\}.
  \end{align*}
  In this case, we clearly have
  \begin{align*}
    [\{{}^{0.5\!}/r\}]_{\Gamma \setminus
      \{\{{}^{0.5\!}/r\} \Rightarrow \{p,{}^{0.5\!}/q,{}^{0.5\!}/r\}\}} = 
    [\{{}^{0.5\!}/r\}]_{\{\{\} \Rightarrow \{p\}\}} = 
    \{p,{}^{0.5\!}/r\},
  \end{align*}
  which means that for $\lexlt$ satisfying~\eqref{eqn:lexlt}
  we must have $\{\} \lexlt \{{}^{0.5\!}/r\}$. On the other hand,
  we also have
  \begin{align*}
    [\{\}]_{\Gamma \setminus \{\{\} \Rightarrow \{p\}\}} &=
    [\{\}]_{\{\{{}^{0.5\!}/r\} \Rightarrow \{p,{}^{0.5\!}/q,{}^{0.5\!}/r\}\}} =
    \{{}^{0.5\!}/p\},
  \end{align*}
  i.e., $\{{}^{0.5\!}/r\} \lexlt \{\}$.
  Hence, $\lexlt$ cannot be a strict total order.
  Analogous counterexamples may also be found using other structures of degrees,
  including the three-element G\"odel chain with ${}^*$ being the identity.
\end{example}

\begin{theorem}\label{th:witnr_glob}
  Let $\mathbf{L}$ be a complete residuated lattice with globalization.
  Then, for each finite non-redundant set of FAIs with saturated
  consequents there is an equivalent non-redundant set of FAIs with
  witnessed non-redundancy.
\end{theorem}
\begin{proof}
  Let $\Gamma$ be a non-redundant set of FAIs with saturated consequents.
  Recall that in this case, each $A \Rightarrow B \in \Gamma$ is in fact
  in the form $A \Rightarrow [A]_\Gamma$. In addition, since $\Gamma$ is
  non-redundant, we have $A \ne [A]_\Gamma$ whenever 
  $A \Rightarrow [A]_\Gamma \in \Gamma$, i.e.,
  it follows that $\Gamma \subseteq \univ$. Now, put
  \begin{align}
    \Sigma =
    \bigl\{[A]_{\Gamma \setminus \{A \Rightarrow [A]_\Gamma\}} \Rightarrow
    [A]_\Gamma;\, A \Rightarrow [A]_\Gamma \in \Gamma\bigr\}.
    \label{eqn:Sigma_infl}
  \end{align}
  First, we prove that $\Sigma$ and $\Gamma$ are equivalent.
  According to Proposition~\ref{pr:equivalence}, it means showing
  that $\mathrm{Mod}(\Gamma) = \mathrm{Mod}(\Sigma)$. Clearly,
  if $M \in \mathrm{Mod}(\Gamma)$ and
  for $A \Rightarrow [A]_\Gamma \in \Gamma$
  we have $[A]_{\Gamma \setminus \{A \Rightarrow [A]_\Gamma\}} \subseteq M$,
  then $A \subseteq M$ because of the extensivity of
  $[{\cdots}]_{\Gamma \setminus \{A \Rightarrow [A]_\Gamma\}}$
  which directly gives $[A]_\Gamma \subseteq M$ on account of
  $M \in \mathrm{Mod}(\Gamma)$. As a consequence,
  we have $\mathrm{Mod}(\Gamma) \subseteq \mathrm{Mod}(\Sigma)$. Thus, it remains
  to prove the converse inclusion.

  Take any $M \in \mathrm{Mod}(\Sigma)$. Applying Lemma~\ref{le:lexlt},
  we assume that $\lexlt$ is the strict total order satisfying~\eqref{eqn:lexlt}.
  Furthermore, for a given $A \Rightarrow [A]_\Gamma \in \Gamma$,
  we assume that $||B \Rightarrow [B]_\Gamma||_M = 1$ holds
  for all $B \lexlt A$. Using Lemma~\ref{le:lexlt} and the fact that ${}^*$
  is globalization~\cite{BeVy:Pmfai}, we may write
  \begin{align}
    [A]_{\Gamma \setminus \{A \Rightarrow [A]_\Gamma\}} =
    A \cup [B_1]_\Gamma \cup \cdots \cup [B_n]_\Gamma
    \label{eqn:clos_union}
  \end{align}
  for $B_1 \lexlt \cdots \lexlt B_n \lexlt A$ so that for the elements
  in the sequence, we have
  $B_1 \subseteq A$,
  $B_2 \subseteq A \cup [B_1]_\Gamma$,
  $B_3 \subseteq A \cup [B_1]_\Gamma \cup [B_2]_\Gamma$,\,\ldots,
  i.e., in general
  \begin{align*}
    \textstyle B_i \subseteq A \cup \bigcup\{[B_k]_\Gamma;\, k < i\}
  \end{align*}
  for all $i=1,\ldots,n$. Now, suppose that $A \subseteq M$.
  It is easy to see that by induction over $i=1,\ldots,n$,
  it follows that 
  \begin{align*}
    \textstyle B_i \subseteq A \cup \bigcup\{[B_k]_\Gamma;\, k < i\} \subseteq M
  \end{align*}
  for all $i=1,\ldots,n$. Since for each
  $B_i \Rightarrow [B_i]_\Gamma \in \Gamma$ we have $B_i \lexlt A$
  and for such a formula we have assumed $||B_i \Rightarrow [B_i]_\Gamma||_M = 1$,
  the previous inclusion gives $[B_i]_\Gamma \subseteq M$ for all $i=1,\ldots,n$.
  Therefore, from \eqref{eqn:clos_union} it follows that
  \begin{align*}
    [A]_{\Gamma \setminus \{A \Rightarrow [A]_\Gamma\}} \subseteq M.
  \end{align*}
  Furthermore, \eqref{eqn:Sigma_infl} and $M \in \mathrm{Mod}(\Sigma)$
  together with the previous inclusion yield $[A]_\Gamma \subseteq M$
  which proves $||A \Rightarrow [A]_\Gamma||_M = 1$.
  Since $A \Rightarrow [A]_\Gamma$ was taken as an arbitrary formula
  in $\Gamma$, we get $M \in \mathrm{Mod}(\Gamma)$.
  Hence, $\Gamma$ and $\Sigma$ are equivalent.

  We now show that $\Sigma$ is non-redundant and its non-redundancy
  is witnessed. In order to see that, we check the condition in
  Lemma~\ref{le:witred}. Thus, take an arbitrary
  $[A]_{\Gamma \setminus \{A \Rightarrow [A]_\Gamma\}} \Rightarrow
  [A]_\Gamma \in \Sigma$. Notice that
  $[A]_{\Gamma \setminus \{A \Rightarrow [A]_\Gamma\}}$ is not a model
  of $\Gamma$ because
  $[A]_{\Gamma \setminus \{A \Rightarrow [A]_\Gamma\}} \subseteq [A]_{\Gamma}$
  and the non-redundancy of $\Gamma$ yields
  $||A \Rightarrow [A]_\Gamma||_{\Gamma \setminus \{A \Rightarrow [A]_\Gamma\}} < 1$,
  i.e., $[A]_{\Gamma} \nsubseteq
  [A]_{\Gamma \setminus \{A \Rightarrow [A]_\Gamma\}}$
  and therefore we get that
  $[A]_{\Gamma \setminus \{A \Rightarrow [A]_\Gamma\}} \subset
  [A]_{\Gamma}$, i.e., $\Sigma \subseteq \univ$.
  Now, suppose that
  \begin{align*}
    [B]_{\Gamma \setminus \{B \Rightarrow [B]_\Gamma\}} \subseteq
    [A]_{\Gamma \setminus \{A \Rightarrow [A]_\Gamma\}}
  \end{align*}
  for any
  $[B]_{\Gamma \setminus \{B \Rightarrow [B]_\Gamma\}} \Rightarrow
  [B]_\Gamma \in \Sigma$ such that $B \ne A$.
  As an immediate consequence of the extensivity of
  $[{\cdots}]_{\Gamma \setminus \{B \Rightarrow [B]_\Gamma\}}$, we get 
  \begin{align*}
    B \subseteq [A]_{\Gamma \setminus \{A \Rightarrow [A]_\Gamma\}}.
  \end{align*}
  Moreover, $B \Rightarrow [B]_\Gamma \in \Gamma$ and since $A \ne B$,
  $[A]_{\Gamma \setminus \{A \Rightarrow [A]_\Gamma\}}$
  is a model of all FAIs in $\Gamma$ with the exception of 
  $A \Rightarrow [A]_\Gamma$, i.e., including 
  $B \Rightarrow [B]_\Gamma \in \Gamma$. Therefore, it follows that
  \begin{align*}
    [B]_\Gamma \subseteq [A]_{\Gamma \setminus \{A \Rightarrow [A]_\Gamma\}}.
  \end{align*}
  As a consequence, $[A]_{\Gamma \setminus \{A \Rightarrow [A]_\Gamma\}}$
  is a model of
  \begin{align*}
    \Sigma \setminus 
    \{[A]_{\Gamma \setminus \{A \Rightarrow [A]_\Gamma\}}
    \Rightarrow [A]_\Gamma\}.
  \end{align*}
  Now, apply Lemma~\ref{le:witred}.
\end{proof}

\begin{example}\label{ex:cannot_arbitrary}
  The construction in Theorem~\ref{th:witnr_glob} cannot be extended to
  any $\mathbf{L}$ with arbitrary ${}^*$. For instance,
  let $\mathbf{L}$ be the three-element G\"odel chain with ${}^*$ being
  the identity. Furthermore, let
  \begin{align*}
    \Gamma &= \{
    \{{}^{0.5\!}/p\} \Rightarrow \{{}^{0.5\!}/p,{}^{0.5\!}/q,r\},
    \{p\} \Rightarrow \{p,q,r\}\}.
  \end{align*}
  Then, the corresponding $\Sigma$ given by \eqref{eqn:Sigma_infl} is 
  \begin{align*}
    \Sigma &= \{
    \{{}^{0.5\!}/p,{}^{0.5\!}/q,{}^{0.5\!}/r\} \Rightarrow
    \{{}^{0.5\!}/p,{}^{0.5\!}/q,r\},
    \{p,{}^{0.5\!}/q,r\} \Rightarrow \{p,q,r\}\}
  \end{align*}
  because, using the fact that $\otimes$ is $\wedge$, we have
  \begin{align*}
    [\{{}^{0.5\!}/p\}]_{\{\{p\} \Rightarrow \{p,q,r\}\}}
    &=
    \{{}^{0.5\!}/p,{}^{0.5\!}/q,{}^{0.5\!}/r\},
    \\
    [\{p\}]_{\{\{{}^{0.5\!}/p\} \Rightarrow \{{}^{0.5\!}/p,{}^{0.5\!}/q,r\}\}}
    &=
    \{p,{}^{0.5\!}/q,r\}.
  \end{align*}
  Trivially, $\{{}^{0.5\!}/p\} \in \mathrm{Mod}(\Sigma)$ because
  \begin{align*}
    ||\{{}^{0.5\!}/p,{}^{0.5\!}/q,{}^{0.5\!}/r\} \Rightarrow
    \{{}^{0.5\!}/p,{}^{0.5\!}/q,r\}||_{\{{}^{0.5\!}/p\}}
    &= 0 \rightarrow 0 = 1,
    \\
    ||\{p,{}^{0.5\!}/q,r\} \Rightarrow
    \{p,q,r\}\}||_{\{{}^{0.5\!}/p\}}
    &= 0 \rightarrow 0 = 1.
  \end{align*}
  In contrast, $\{{}^{0.5\!}/p\} \not\in \mathrm{Mod}(\Gamma)$ since
  \begin{align*}
    ||\{{}^{0.5\!}/p\} \Rightarrow \{{}^{0.5\!}/p,{}^{0.5\!}/q,r\}
    ||_{\{{}^{0.5\!}/p\}} &=
    1 \rightarrow 0 = 0.
  \end{align*}
  Hence, using Proposition~\ref{pr:equivalence},
  $\Gamma$ and $\Sigma$ are not equivalent.
\end{example}

The following assertion gives a connection between sets of FAIs with witnessed
non-redundancy and systems of pseudo-intents, see~\eqref{eqn:system_P}.
As a consequence, under the assumption of ${}^*$ being the globalization,
we establish a procedure for getting a non-redundant base given by pseudo-intents
from any non-redundant base of a given $\mathbf{L}$-context.

\begin{theorem}\label{th:wit=psd}
  Let\/ $\mathbf{I} = \langle X,Y,I\rangle$ be a formal\/ $\mathbf{L}$-context,
  $\Sigma$ be a non-redundant base of\/ $\mathbf{I}$. Then, the following
  conditions are equivalent:
  \begin{enumerate}\parskip=0pt
  \item[\itm{1}]
    $\Sigma \subseteq \univ[\Sigma]$ and the non-redundancy of
    $\Sigma$ is witnessed.
  \item[\itm{2}]
    There is a system $\mathcal{P}$ of pseudo-intents of\/ $\mathbf{I}$
    such that $\Sigma$ is given by~\eqref{eqn:nred}.
  \end{enumerate}
\end{theorem}
\begin{proof}
  Clearly, if $\Sigma$ is given by~\eqref{eqn:nred} for
  some system $\mathcal{P}$ of pseudo-intents of $\mathbf{I}$,
  then $\Sigma \subseteq \univ[\Sigma]$. Furthermore, since $\Sigma$
  is complete in $\mathbf{I}$, we have $M^{\downarrow\uparrow} = [M]_\Sigma$
  for all $M \in L^Y$, see \cite[Theorem~5.3]{BeVy:ADfDwG}. Now, assume that
  \itm{1} holds. Since $\Sigma$ is complete in $\mathbf{I}$,
  Theorem~\ref{th:compl} yields that
  for each $P \ne P^{\downarrow\uparrow}$, we have 
  $P \Rightarrow P^{\downarrow\uparrow} \in \Sigma$ provided that
  $P \in \mathrm{Mod}(\Sigma \setminus \{P \Rightarrow P^{\downarrow\uparrow}\})$.
  In addition to that, the fact that the non-redundancy of $\Sigma$ is witnessed
  gives that for each $P \ne P^{\downarrow\uparrow}$, we have 
  $P \in \mathrm{Mod}(\Sigma \setminus \{P \Rightarrow P^{\downarrow\uparrow}\})$
  provided that $P \Rightarrow P^{\downarrow\uparrow} \in \Sigma$. Altogether,
  for each $P \ne P^{\downarrow\uparrow}$, we have 
  \begin{align*}
    P \Rightarrow P^{\downarrow\uparrow} \in \Sigma
    \text{ if{}f }
    P \in \mathrm{Mod}(\Sigma \setminus \{P \Rightarrow P^{\downarrow\uparrow}\}).
  \end{align*}
  Thus, for $\mathcal{P} = \{P \in L^Y;\,
  P \Rightarrow P^{\downarrow\uparrow} \in \Sigma\}$, it follows
  that for each $P \ne P^{\downarrow\uparrow}$:
  \begin{align*}
    P \in \mathcal{P}
    \text{ if{}f }
    P \in \mathrm{Mod}(\Sigma \setminus \{P \Rightarrow P^{\downarrow\uparrow}\}).
  \end{align*}
  Now, observe that
  $P \in \mathrm{Mod}(\Sigma \setminus \{P \Rightarrow P^{\downarrow\uparrow}\})$
  is true if{}f for each $Q \Rightarrow Q^{\downarrow\uparrow} \in \Sigma$ such
  that $Q \ne P$, we have $||Q \Rightarrow Q^{\downarrow\uparrow}||_P = 1$
  Hence, for each $P \ne P^{\downarrow\uparrow}$:
  \begin{align*}
    P \in \mathcal{P}
    \text{ if{}f }
    ||Q \Rightarrow Q^{\downarrow\uparrow}||_P = 1
    \text{ for any } Q \in \mathcal{P} \text{ such that } Q \ne P.
  \end{align*}
  Altogether, $\mathcal{P}$ is a system of pseudo-intents and $\Sigma$ is
  of the form~\eqref{eqn:nred} which proves \itm{2}. Conversely,
  \itm{1} is a direct consequence of \itm{2}.
\end{proof}

\begin{corollary}\label{cor:getbase}
  Let\/ $\mathbf{I} = \langle X,Y,I\rangle$ be a formal\/ $\mathbf{L}$-context
  and let $\Gamma$ be a non-redundant base of $\mathbf{I}$ which consists of
  FAIs with saturated consequents. If $\mathbf{L}$ is finite and ${}^*$
  is globalization, then $\Sigma$ given by~\eqref{eqn:Sigma_infl} is
  a minimal base of\/ $\mathbf{I}$ and
  \begin{align}
    \mathcal{P} = \{A \in L^Y\!;\, A \Rightarrow B \in \Sigma\}
    \label{eqn:P}
  \end{align}
  is a system of pseudo-intents of $\mathbf{I}$.
\end{corollary}
\begin{proof}
  Apply Theorem~\ref{th:witnr_glob} and Theorem~\ref{th:wit=psd}.
  The minimality of $\Sigma$ then follows
  by~\cite[Theorem~5.20]{BeVy:ADfDwG}, cf. also~\cite{BeVy:Falaitvenb}.
\end{proof}

\begin{remark}
  Corollary~\ref{cor:getbase} allows us to find a non-redundant base of $\mathbf{I}$
  given by a system of pseudo-intents in an alternative way. The method is
  restricted to ${}^*$ being globalization but as we shall see in
  Section~\ref{sec:notes}, the procedure may produce the desired base even in
  case of general hedges and our experiments indicate that this happens frequently.
  Compared to the graph-based method discussed in Section~\ref{sec:bases_survey},
  such an approach is considerably faster.
\end{remark}

\section{Experimental Observations and Comments}\label{sec:notes}
In this section, we present an additional experimental insight into the problem
of computing sets of FAIs with witnessed non-redundancy
based on the observation in the proof of Theorem~\ref{th:witnr_glob}
and Corollary~\ref{cor:getbase}. Recall
that the assertions presuppose that the utilized hedge is a globalization and
Example~\ref{ex:cannot_arbitrary} shows that the procedure cannot be extended
for arbitrary hedges in general. That is, if $\Gamma$ is finite and non-redundant
set of FAIs with saturated consequents, it can happen that $\Sigma$ given by
\eqref{eqn:Sigma_infl} is not equivalent to $\Gamma$. A question is whether
such situations are frequent or rare. The first series of our experiments
focuses on this phenomenon.

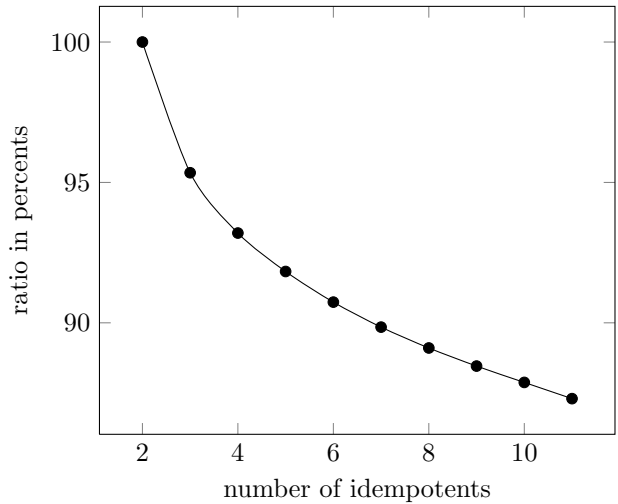
\begin{figure}
  \centering
  \begin{tikzpicture}
    \begin{axis}[
      xlabel={number of idempotents},
      ylabel={ratio in percents}]
      \addplot[smooth,mark=*] plot coordinates {
        (2, 100.0000)
        (3, 95.34444)
        (4, 93.19722)
        (5, 91.82798)
        (6, 90.73730)
        (7, 89.84722)
        (8, 89.10417)
        (9, 88.45972)
        (10, 87.87778)
        (11, 87.30000)};
    \end{axis}
  \end{tikzpicture}
  \caption{Percentage of successful transformations of randomly generated
    non-redundant sets of FAIs with saturated consequents into equivalent
    sets with witnessed non-redundancy.}
  \label{fig:ratio}
\end{figure}

We have performed experiments with randomly generated non-redundant
sets of FAIs with saturated consequents using structures of truth degrees defined
on $11$-element equidistant subchains of the real unit interval. The utilized
structures of degrees were all BL-algebras~\cite{Haj:BL1,Haj:MFL} which
can be defined on such chains.
It is a well-known fact that such BL-algebras result as ordinal
sums~\cite{Haj:BL1} of finite linear \L ukasiewicz algebras and
are given solely by the elements which are idempotent with respect to $\otimes$,
see~\cite{KMP:TN,DeMe:TNPL}. The hedge was always considered as the identity.
Recall that in this setting, $\Sigma$ given by \eqref{eqn:Sigma_infl} is not
equivalent to the input $\Gamma$ in general. Nevertheless, our experimental
observations show that with growing number of idempotents in the structure,
the ratio of successful transformations of $\Gamma$ to an equivalent $\Sigma$ given
by~\eqref{eqn:Sigma_infl} is decreasing. Figure~\ref{fig:ratio} shows the mean
percentages of successful transformations depending on the number of idempotents
in linear $11$-element BL-algebras with ${}^*$ being the identity. The graph
was generated using more than $10^6$ randomly generated sets of FAIs consisting
of $20$ formulas using up to $10$ distinct attributes. Interestingly,
in case of the $11$-element \L ukasiewicz chain (the case of only $2$ idempotents),
the transformation was always successful.
Therefore, we hypothesize that at least on finite \L ukasiewicz chains,
Theorem~\ref{th:witnr_glob} can be extended to ${}^*$ being the identity.
To prove this hypothesis is an interesting open problem. Note that due to
Example~\ref{ex:nlukid}, one cannot use the proof technique of
Theorem~\ref{th:witnr_glob} because the utilized strict total order
described in Lemma~\ref{le:lexlt} may not exist. Also note that even in
the worst case which seems to be the $11$-element G\"odel chain
(all elements idempotent),
the percentage of successes is relatively high (above 87\,\%), i.e.,
we can say that even if Theorem~\ref{th:witnr_glob} does not hold for general
hedges, the chance of obtaining a desired equivalent set of FAIs with
witnessed non-redundancy is relatively high.

\begin{figure}
  \centering
  \begin{tikzpicture}
    \begin{axis}[
      xlabel={density of input datasets in percents},
      ylabel={mean running time in seconds}]
      \addplot[smooth,mark=x] plot coordinates {
        (06, 1.259095238)
        (11, 0.709122105)
        (16, 0.410900923)
        (21, 0.229117177)
        (26, 0.155245586)
        (31, 0.106250645)
        (36, 0.066180776)
        (41, 0.044732894)
        (46, 0.040428098)
        (51, 0.039018297)
        (56, 0.034192170)
        (61, 0.031251061)
        (66, 0.036033878)
        (71, 0.043469881)
        (76, 0.037082771)
        (81, 0.019290560)
        (86, 0.011646647)
        (91, 0.009541485)};
      \addlegendentry{graph-based method}
      \addplot[smooth,mark=*] plot coordinates {
        (06, 0.0008809524)
        (11, 0.0014273684)
        (16, 0.0017538462)
        (21, 0.0020343542)
        (26, 0.0022620299)
        (31, 0.0024660812)
        (36, 0.0027295645)
        (41, 0.0029659246)
        (46, 0.0031263119)
        (51, 0.0032611189)
        (56, 0.0033947347)
        (61, 0.0034816534)
        (66, 0.0035080377)
        (71, 0.0034754619)
        (76, 0.0034273535)
        (81, 0.0033845870)
        (86, 0.0032972973)
        (91, 0.0031572052)};
      \addlegendentry{alternative method}
    \end{axis}
  \end{tikzpicture}
  \caption{Real running time of the graph-based and alternative methods for
    computing bases of FAIs with witnessed non-redundancy.}
  \label{fig:slowfast}
\end{figure}
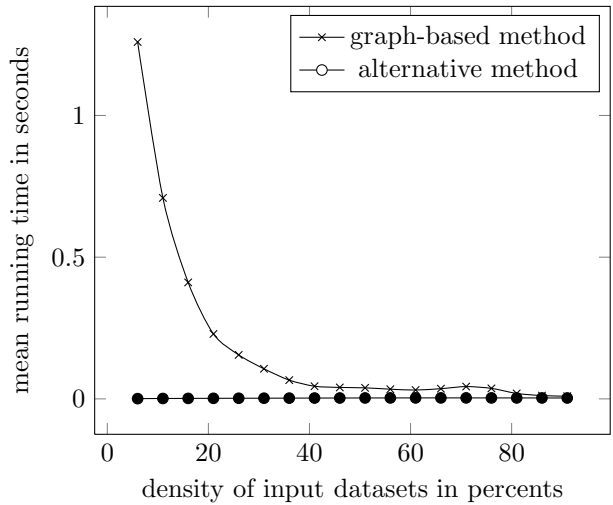

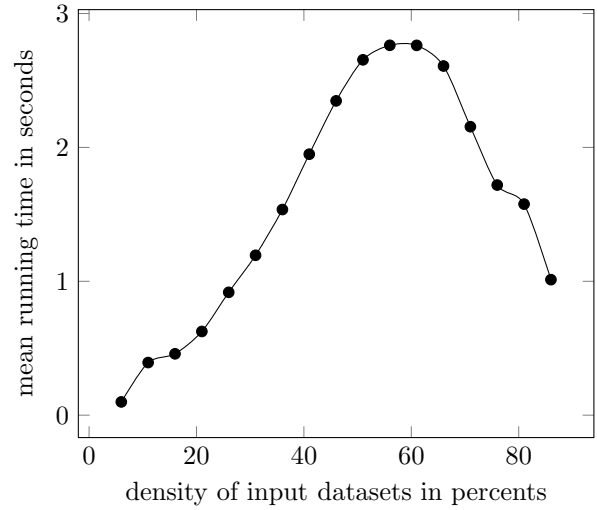
\begin{figure}
  \centering
  \begin{tikzpicture}
    \begin{axis}[
      xlabel={density of input datasets in percents},
      ylabel={mean running time in seconds}]
      \addplot[smooth,mark=*] plot coordinates {
        (06, 0.0990000)
        (11, 0.3930000)
        (16, 0.4582286)
        (21, 0.6251985)
        (26, 0.9176860)
        (31, 1.1944113)
        (36, 1.5360988)
        (41, 1.9494414)
        (46, 2.3476688)
        (51, 2.6542379)
        (56, 2.7628059)
        (61, 2.7622852)
        (66, 2.6080253)
        (71, 2.1547873)
        (76, 1.7188715)
        (81, 1.5764889)
        (86, 1.0122308)};
    \end{axis}
  \end{tikzpicture}
  \caption{Real running time of the alternative method for computing bases
    of FAIs with witnessed non-redundancy.}
  \label{fig:fastTime}
\end{figure}

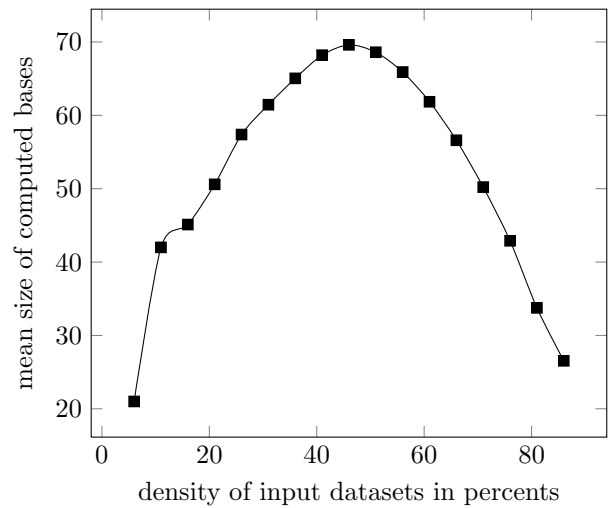
\begin{figure}
  \centering
  \begin{tikzpicture}
    \begin{axis}[
      xlabel={density of input datasets in percents},
      ylabel={mean size of computed bases}]
      \addplot[smooth,mark=square*] plot coordinates {
        (06, 21.00000)
        (11, 42.00000)
        (16, 45.11429)
        (21, 50.59559)
        (26, 57.38259)
        (31, 61.44765)
        (36, 65.04053)
        (41, 68.20434)
        (46, 69.60760)
        (51, 68.60106)
        (56, 65.90500)
        (61, 61.84968)
        (66, 56.60337)
        (71, 50.22086)
        (76, 42.89944)
        (81, 33.75556)
        (86, 26.53846)};
    \end{axis}
  \end{tikzpicture}
  \caption{Mean sizes of bases of FAIs with witnessed non-redundancy
    depending on the density of input data sets.}
  \label{fig:fastSize}
\end{figure}

Our next experiment shows the comparison of running times needed to determine
a non-redundant base given by systems of pseudo-intents using the graph-based
method outlined in~Section~\ref{sec:bases_survey} and the alternative method
based on removing redundant FAIs from a complete set consisting of FAIs with
saturated consequents and then applying Theorem~\ref{th:witnr_glob}.
Figure~\ref{fig:slowfast} shows the alternative method is (as expected)
faster by an order of several magnitudes. The graph was generated
based on observing the running time of the algorithms for $10^5$ randomly generated
contexts with $50$ objects and (only) $4$ attributes using
a three-element \L ukasiewicz chain.
For higher numbers of attributes and/or higher numbers of truth degrees,
the graph-based method is practically not applicable. The graph in
Figure~\ref{fig:slowfast} shows the dependency of the mean running time
on the density of input data sets which is for
$\mathbf{I} = \langle X,Y,I\rangle$ introduced as
\begin{align}
  \cfrac{%
    \textstyle\sum_{x \in X}\sum_{y \in Y}I(x,y)}{%
    |X| \cdot |Y|} \cdot 100.
  \label{eqn:dens}
\end{align}
Notice that Figure~\ref{fig:slowfast} shows that the tendency of the graph-based
algorithm is that for sparse datasets (i.e., datasets where the value
of~\eqref{eqn:dens} is small) the running time is higher than for more dense
datasets. This is caused by the fact that in such cases, the graphs associated
to input data are usually more complex. On the contrary,
Figure~\ref{fig:fastTime} and Figure~\ref{fig:fastSize} show that
in the case of the alternative algorithm, the running time more or less copies
the size of the computed bases (in terms of the number of formulas contained
in the bases). This behavior may be considered more natural. For both,
we have used the same parameters: $\mathbf{L}$~with $5$ truth degrees
and $10^4$ randomly generated datasets with $10$ objects and $10$ attributes.

To sum up, the experiments presented in this section indicate that
(i) the method of computing bases with witnessed non-redundancy presented in
this paper can be applied even if ${}^*$ is not a globalization and its success
rate is relatively high, and (ii) the method significantly outperforms
the graph-based method. As we have mentioned in the section, further investigation
regarding the existence of systems of pseudo-intents for general hedges and
possible generalizations of Theorem~\ref{th:witnr_glob} are needed and
we consider these important open problems.

\subsubsection*{Acknowledgment}
Supported by grant no. \verb|P202/14-11585S| of the Czech Science Foundation.


\footnotesize
\bibliographystyle{amsplain}
\bibliography{osgaiwnr}

\end{document}